\documentclass[twoside]{article}
\newif\ifsup\suptrue
 \usepackage[accepted]{aistats2020}
\usepackage[round]{natbib}

\usepackage{amsmath}
\usepackage{amssymb}
\usepackage{amsthm}
\usepackage{mathtools}
\usepackage{tikz}

\usetikzlibrary{arrows,decorations.markings}
\usetikzlibrary{calc}

\usepackage[linesnumbered,ruled,vlined]{algorithm2e} 
\SetKw{KwInit}{Initialize}
\SetKwProg{KwUpdate}{determine}{}{}
\usepackage[textsize=tiny]{todonotes}
\usepackage[normalem]{ulem}

\newtheorem{theorem}{Theorem}
\newtheorem{definition}{Definition}
\newtheorem{remark}{Remark}

\newtheorem{lemma}{Lemma}

\newtheorem{claim}{Fact}

\usepackage{hyperref}
\definecolor{dkblue}{cmyk}{1,.54,.04,.19} 
\hypersetup{
      bookmarks=true,         
      unicode=false,          
      pdftoolbar=true,        
      pdfmenubar=true,        
      pdffitwindow=false,     
      pdfstartview={FitH},    
      pdftitle={Bandits},    
      pdfauthor={},     
      pdfsubject={Bandits},   
      pdfcreator={pdflatex},   
      pdfproducer={Producer}, 
      pdfkeywords={bandits} {statistics} {machine learning}, 
      pdfnewwindow=true,      
      colorlinks=true,        
      linkcolor=dkblue,       
      citecolor=dkblue,       
      filecolor=dkblue,       
      urlcolor=dkblue,        
}
\usepackage[capitalize]{cleveref}
\crefname{claim}{Fact}{Facts}

\newcommand{\argmax}{\operatornamewithlimits{arg\,max}}
\newcommand{\argmin}{\operatornamewithlimits{arg\,min}}

\newcommand{\R}{\mathbb{R}}
\newcommand{\E}{\mathbb{E}}
\newcommand{\e}{\mathbf{e}}
\newcommand{\ones}{\mathbf{1}_k}
\newcommand{\D}{\operatorname{D}}

\newcommand{\ip}[1]{\langle #1 \rangle}

\newcommand{\norm}[1]{\Vert #1 \Vert}
\newcommand{\Reg}{\mathfrak{R}}

\newcommand{\bbI}{\mathbb I}

\newcommand{\dom}{\operatorname{dom}}

\let\epsilon\varepsilon

\usepackage{arrayjobx}
\usepackage{multido}
\newarray\Eqnotes

\newcounter{equationnotes}
\stepcounter{equationnotes}
\newcounter{writtennotes}
\stepcounter{writtennotes}

\newcommand{\eqnote}[1]{\Eqnotes(\theequationnotes)={#1}\mbox{(\alph{equationnotes})}\stepcounter{equationnotes}}
\newcommand{\writeeqnotes}{\addtocounter{equationnotes}{-\thewrittennotes}\multido{\iActor=\thewrittennotes+1}{\value{equationnotes}}{%
(\alph{writtennotes}) \Eqnotes(\iActor)\stepcounter{writtennotes}}\setcounter{equationnotes}{\thewrittennotes}\setcounter{writtennotes}{\theequationnotes}}
\newcommand{\resetnotes}{\setcounter{equationnotes}{1}\setcounter{writtennotes}{1}}

\newcommand{\lr}[1]{\left (#1\right)}
\newcommand{\lrc}[1]{\left \{#1\right\}}

\begin{document}

%

%

\twocolumn[

\runningtitle{An Optimal Algorithm for Adversarial Bandits with Arbitrary Delays}
\aistatstitle{An Optimal Algorithm for Adversarial Bandits\\ with Arbitrary Delays}

\aistatsauthor{ Julian Zimmert \And Yevgeny Seldin }

\aistatsaddress{ University of Copenhagen } 
]

\begin{abstract}
We propose a new algorithm for adversarial multi-armed bandits with unrestricted delays. The algorithm is based on a novel hybrid regularizer applied in the Follow the Regularized Leader (FTRL) framework. It achieves $\mathcal{O}(\sqrt{kn}+\sqrt{D\log(k)})$ regret guarantee, where $k$ is the number of arms, $n$ is the number of rounds, and $D$ is the total delay. The result matches the lower bound within constants and requires no prior knowledge of $n$ or $D$. Additionally, we propose a refined tuning of the algorithm, which achieves $\mathcal{O}(\sqrt{kn}+\min_{S}(|S|+\sqrt{D_{\bar S}\log(k)}))$ regret guarantee, where $S$ is a set of rounds excluded from delay counting, $\bar S = [n]\setminus S$ are the counted rounds, and $D_{\bar S}$ is the total delay in the counted rounds. If the delays are highly unbalanced, the latter regret guarantee can be significantly tighter than the former. The result requires no advance knowledge of the delays and resolves an open problem of \citet{TCS19}. The new FTRL algorithm and its refined tuning are anytime and require no doubling, which resolves another open problem of \citet{TCS19}.
\end{abstract}

\section{Introduction}

Multi-armed bandits are a fundamental sequential decision making problem with an increasing number of industrial applications.
In the multi-armed bandit setting, a learner repeatedly chooses an action from a finite set of actions and immediately observes a loss for that specific action.
The action might be, for example, a choice of an advertisement layout out of a finite set of layouts. 
The loss could be the response of a user to the layout, for example, a lack of a click on the advertisement. 
In practice, it is often required to make decisions for new users before observing the feedback of all previous users, 
either due to response latency or parallel interaction with multiple users.
This can be modeled by introducing a \emph{delay} between the action and observation.

\begin{table*}[t]
\caption{Overview of state-of-the-art regret bounds for multi-armed bandits with delayed feedback. (*) requires oracle knowledge of the time horizon $n$ and the total delay $D$; the result appeared independently in two papers. (**) requires advance knowledge of the delays $d_t$ ``at action time'' $t$.}
\label{tbl:summary}
\centering
\smallskip
\begin{tabular}{l l l l}
\hline
Setting & Regret upper and lower bounds && Reference\\
\hline
Uniform delays $d$ &$\Omega(\max\{\sqrt{kn},\sqrt{dn\log(k)}\})$&& \citet{NGMM16}\\
                   &$\mathcal{O}(\sqrt{kn\log(k)}+\sqrt{dn\log(k)})$ & &\citet{NGMM16}\\
                   &$\mathcal{O}(\sqrt{kn}+\sqrt{dn\log(k)})$ & &This paper\\
\hline
Arbitrary delays,  &$\mathcal{O}(\sqrt{kn\log(k)}+\sqrt{D\log(k)})$ & (*) &
$\begin{cases}\text{\citet{TCS19}}\\
\text{\citet{BZCBB19}} \end{cases}$\\
non-adaptive bounds &$\mathcal{O}(\sqrt{k^2n\log(k)}+\sqrt{D\log(k)})$ &&\citet{BZCBB19} \\&$\mathcal{O}(\sqrt{kn}+\sqrt{D\log(k)})$ && This paper\\
\hline
Arbitrary delays, &$\mathcal{O}(\min_\beta |S_\beta|+\beta\log(k)+\beta^{-1}(kn+D_{\bar S_\beta}))$ & (**) & \citet{TCS19}\\
adaptive bounds   &$\mathcal{O}(\sqrt{kn}+\min_{S}(|S|+\sqrt{D_{\bar S}\log(k)}))$ && This paper\\
\end{tabular}
\end{table*}

We focus on the oblivious adversarial (a.k.a.\ non-stochastic) bandit setting, meaning that the sequence of losses and the delays are fixed before the start of the game. 
The setting was first studied by \citet{NGMM16} under the assumption of uniform delays, which are all equal to $d$. 
They proved a lower bound of $\Omega(\max\{\sqrt{kn},\sqrt{dn\log(k)}\})$ for $d\leq n/\log(k)$ (they do not report the $\log(k)$ term)
and an almost matching upper bound of $\mathcal{O}(\sqrt{kn\log(k)}+\sqrt{dn\log(k)})$. 
By translating individual delays into the total delay $D = dn$ the lower bound for uniform delays is $\Omega(\max\{\sqrt{kn},\sqrt{D\log(k)}\})$.
\citet{TCS19} and \citet{BZCBB19} independently derived an algorithm that can handle non-uniform delays and achieves an $\mathcal{O}(\sqrt{kn\log(k)}+\sqrt{D\log(k)})$ regret bound under the assumption that $n$ and $D$ are known in advance.
\citeauthor{TCS19} further provide a doubling scheme that achieves the same regret bound under the assumption that the delays $d_t$ are known "at action time", i.e., at time $t$, but $n$ and $D$ are unknown, whereas
\citeauthor{BZCBB19} provide a doubling scheme that achieves an $\mathcal{O}(\sqrt{k^2n\log(k)}+\sqrt{D\log(k)})$ regret bound when $n$ and $D$ are unknown and the delays $d_t$ are observed together with the observations, i.e., at time $t+d_t$.

\citeauthor{TCS19} further observe that if the delays are highly unbalanced it may be worth ``skipping'' rounds with excessively large delays. 
``Skipping'' means that the regret in the corresponding round is trivially bounded by 1 and the observation is ignored by the algorithm. 
The skipping approach of \citeauthor{TCS19}\ requires knowledge of the delays ``at action time''. 
Under the assumption that this information is available, \citeauthor{TCS19}\ provide an algorithm that achieves $\mathcal{O}(\min_\beta |S_\beta|+\beta\log(k)+\beta^{-1}(kn+D_{\bar S_\beta}))$ regret guarantee, where $\beta$ is the skipping threshold (the rounds with delays $d_t \geq \beta$ are skipped), $S_\beta$ is the set of skipped rounds and $|S_\beta|$ is their number, $\bar S_\beta = [n] \setminus S_\beta$ are the remaining rounds (where $[n] = \lrc{1,\dots,n}$), and $\D_{\bar S_\beta} = \sum_{t \in \bar S_\beta} d_t$ is their total delay. 
\citeauthor{TCS19}\ provide an example, where the first $\lfloor\sqrt{kn/\log(k)}\rfloor$ rounds have delays of order $n$ and the remaining rounds have zero delays. 
By skipping the first rounds, the dependence of the regret bound on $n$ improves from order $n^{3/4}$ to $n^{1/2}$. 
The skipping procedure of \citeauthor{TCS19}\ crucially depends on availability of delays ``at action time'' in order to make the skipping decision and the skipping threshold $\beta$ is tuned by doubling. 
Relaxation of the assumption on early availability of delays, as well as replacement of doubling with anytime strategies (i.e., algorithms without resets) were left as open questions.

We resolve both open questions and make the following contributions:
\begin{enumerate}
	\item We provide an anytime FTRL algorithm based on a novel hybrid regularizer. The regularizer combines $\frac{1}{2}$-Tsallis entropy and negative entropy, each with its own learning rate. The algorithm requires no advance knowledge of the delays and  
 achieves a regret bound of $\mathcal{O}(\sqrt{kn}+\sqrt{D\log(k)})$, which matches the lower bound within constants.
	\item We provide a novel ``skipping'' technique, which allows to ``ignore'' rounds with excessively large delays with no advance knowledge of the delays. We put ``skipping'' and ``ignore'' in quotation marks, because the observations are still used by the algorithm and the ``skipped'' rounds are only excluded from updates of the learning rate. 
We prove an $\mathcal{O}(\sqrt{kn}+\min_{S}(|S|+\sqrt{D_{\bar S}\log(k)}))$ regret bound for the refined algorithm. The bound is slightly tighter than the refined regret bound of \citet{TCS19}, but most importantly it requires no advance knowledge of the delays. 
\footnote{We note that the new skipping technique could also be combined with the doubling scheme of \citeauthor{TCS19} to eliminate the need in advanced knowledge of delays there as well. However, the anytime FTRL algorithm presented here is much more elegant than doubling.}
\end{enumerate}

In \cref{tbl:summary} we provide a comparison of state-of-the art bounds with our new results. Additional prior work in other online learning settings with delayed feedback includes the full information setting studied by \citet{JGS16}, who derived a general reduction to a non-delayed problem.
To the best of our knowledge, no similar reduction under bandit feedback has been found yet. Another related setting are bandits with anonymous composite feedback, where the learner is not informed about the round from which the delayed observation is coming from, neither the identity of the action it corresponds to, and delayed observations from several rounds may be composed together with no possibility to separate them. This harder setting was studied by \citeauthor{CGM18}, who derived an $\mathcal{O}(\sqrt{kd_{max}n\log(k)})$ regret bound, where $d_{max}$ is a known upper bound on the delays. We refer the reader to \citet{TCS19} for further review of prior work in related settings. 

The paper is structured in the following way:
\cref{sec:problem} provides a formal definition of the problem setting.
\cref{sec:algorithm} explains in detail our algorithm and two versions of learning rate tuning.
\cref{sec:main results} contains our main theorems, as well as an intuition behind the refined learning rate tuning.
\cref{sec:proofs} presents a general analysis of FTRL for multi-armed bandits with delays and formally proves the theorems from the previous section.
Finally, \cref{sec:conclusion} provides a summary and directions for future work.
%

\section{Problem setting}
\label{sec:problem}
Adversarial bandits with delay is a sequential game between a learner and an environment with $k$ fixed actions.
At time steps $t=1,\dots,n$ the learner picks actions $A_t\in[k]$ and immediately suffers the loss $\ell_{t,A_t}$, where $(\ell_t)_{t=1,\dots,n}$ are vectors in $[0,1]^k$.
Unlike in the regular bandit problem, the learner does not necessarily observe the loss $\ell_{t,A_t}$ at the end of round $t$.
Instead, the environment chooses a sequence of delays $(d_t)_{t=1,\dots,n}$ and the player observes the tuples $(s,\ell_{s,A_s})$ for each $s$ such that $s+d_s = t$ at the end of round $t$. 
Without loss of generality, we assume that all outstanding tuples are observed at the end of the game, i.e., $t+d_t \leq n$ for all $t$. 
We focus on the oblivious adversarial setting (sometimes called ``non-stochastic''), which means that both the sequence of losses $(\ell_t)_{t=1,\dots,n}$ and the sequence of delays $(d_t)_{t=1,\dots,n}$ are chosen by the environment at the beginning of the game. 
We use $D = \sum_{t=1}^nd_t$ to denote the total delay. 
The learner has no prior knowledge of the quantities $n, D$, or $(d_t)_{t=1,\dots,n}$. 
The performance of the algorithm is measured by its expected regret
\begin{align*}
\Reg_n := \E\left[\sum_{t=1}^n\ell_{t,A_t}\right]-\min_{i\in[k]}\sum_{t=1}^n\ell_{t,i}\,.
\end{align*}

\paragraph{Some technical definitions} We use $\Delta([k])=\{x\in\mathbb{R}_+^k|\norm{x}_1=1\}$ to denote the $(k-1)$-simplex. For a set $S\subset [n]=\{1,\dots,n\}$, we denote its complement by $\bar S=[n]\setminus S$. For a convex function $F$ we use $F^*$ to denote its convex conjugate (a.k.a.\ Fenchel conjugate) and $\overline F^*$ to denote the constrained convex conjugate. They are defined, respectively, by
\begin{align*}
&F^*(y) = \max_{x\in\R^k}\ip{x,y} - F(x),\\
&\overline{F}^*(y) = \max_{x\in\Delta([k])}\ip{x,y} - F(x)\,.
\end{align*}

\section{Algorithm}
\label{sec:algorithm}
Our \cref{alg:FTRL-delay}  is a standard Follow the Regularized Leader (FTRL) algorithm that works with importance weighted loss estimators of all observations available up to the current point in time.
The loss estimators are defined by
\[\hat\ell_s = \frac{\ell_{s,A_s}}{x_{s,A_s}}\e_{A_s}\,,\]
where $x_{s,A_s}$ is the algorithm's probability of selecting action $A_s$ at round $s$ and $\e_{A_s}$ is a standard basis vector. 
We define the cumulative observed loss estimator at time $t$ by
\begin{align*}
\hat L^{obs}_t = \sum_{s:s+d_s<t}\hat\ell_s\,.
\end{align*}
Given a convex regularizer $F_t : \R^k\rightarrow \R$, FTRL samples action $A_t$ according to the distribution
\begin{align*}
x_t = \argmin_{x\in\Delta([k])} \ip{x,\hat L^{obs}_t} + F_t(x)\,.
\end{align*}
$x_t$ can be equivalently expressed as $x_t = \nabla \overline F_t^*(-\hat L^{obs}_t)$.

We are using a hybrid regularizer $F_t = F_{t,1} + F_{t,2}$, where in contrast to most prior work each of the two parts of the regularizer has its own learning rate. 
\begin{align*}
    \underbrace{F_t(x)}_{=\sum_{i=1}^kf_t(x_i)} = \underbrace{-\sum_{i=1}^k2\sqrt{t}x_i^{1/2}}_{F_{t,1}(x)=\sum_{i=1}^kf_{t,1}(x_i)} + \underbrace{\eta_t^{-1}\sum_{i=1}^kx_i\log(x_i)}_{F_{t,2}(x)=\sum_{i=1}^kf_{t,2}(x_i)}\,.
\end{align*}
The first part of the regularizer $F_{t,1}(x) = \sqrt{t} F_1(x)$ is the $\frac{1}{2}$-Tsallis entropy $F_1(x) = -2\sum_{i=1}^k \sqrt{x_i}$ with learning rate $\frac{1}{\sqrt t}$, which is non-adaptive to the problem. The second part of the regularizer $F_{t,2}(x) = \eta_t^{-1}F_2(x)$ is the negative entropy $F_2(x) = \sum_{i=1}^k x_i \log(x_i)$ with adaptive learning rate $\eta_t$. We call a sequence of learning rates $(\eta_t)_{t=1,\dots,n}$ \emph{proper} if it is non-increasing and can be defined using information available at the beginning of round $t$.

\subsection{Intuition behind the regularizer}
Hybrid regularizers have been successfully used in adaptive regret bounds for sparse bandits, online portfolio selection, adversarially robust semi-bandits, and adaptive first order bounds for multi-armed bandits \citep{BCL17, LWZ18, ZLW19,PL19}. 
They are useful for targeting multiple objectives. 
In our case, the regret lower bound for bandits with fixed delay $d$ is $\Omega(\max\{\sqrt{kn}, \sqrt{dn\log(k)}\})$ \citep{NGMM16}. The first part of the bound is the standard regret lower bound for multi-armed bandits with no delays, which is clearly also a lower bound for the problem with delays. The second part of the bound is achieved by grouping the game rounds into batches of size $d$ and reducing the game to a full information game over $n/d$ rounds with loss range $[0,d]$. The second part is then a lower bound on the regret in the full information game.

Our regularizer uses the same decomposition of the problem.
We combine the optimal regularizer for the standard bandit problem with no delay, the $\frac{1}{2}$-Tsallis Entropy, with the optimal regularizer for the full information problems, the negative entropy. We further tune the learning rate for the second part to the actual delay sequence $(d_t)_{t=1,\dots,n}$. 

\subsection{Tuning of the learning rate}
We propose and analyze two versions of learning rate tuning.
The \emph{simple tuning} is given in \cref{alg:FTRL-delay}. For \emph{advanced tuning}, replace the colored blocks {\bf Initialize} and {\bf determine} \emph{$\eta_t$} in \cref{alg:FTRL-delay} with the corresponding blocks from \cref{alg:eta}.

\begin{algorithm}
\caption{FTRL for bandits with delay}
\label{alg:FTRL-delay}
\DontPrintSemicolon
\LinesNumberedHidden
 
\KwIn{Proper learning rate rule $\eta_t$}
\KwInit{$\hat L_1^{obs}=0$}\;
{\color{blue}\KwInit{$\mathfrak{D}_0=0$~~~~~~~~~~~~~~~~~~~~\emph{(simple tuning)}}}\;
\For{$t= 1,\ldots,n$}{
{\color{violet}
\KwUpdate{$\eta_t$}{
$\begin{rcases*}
\mbox{Set }\mathfrak{D}_t = \mathfrak{D}_{t-1} + \mathfrak{d}_t \\
\mbox{Set }\eta_t^{-1} = \sqrt{2\mathfrak{D}_t/\log(k)}
\end{rcases*} \mbox{(simple tuning)}$\;
}
}
Set $x_t = \argmin_{x\in\Delta([k])}\ip{x,\hat L^{obs}_t}+F_t(x) $\;
Sample $A_t \sim x_t$ \;
\For{$s: s+d_s = t$}{
    Observe $(s,\ell_{s,A_s})$\;
    Construct $\hat\ell_s$ and update $\hat L_t^{obs}$\;
}
}
\end{algorithm}

\paragraph{Simple tuning}
We define the key quantity, which is used for tuning the learning rate.
\begin{definition}
The \emph{number of outstanding observations} at round $t$ is defined by
\[
\mathfrak{d}_t = \sum_{s=1}^{t-1} \bbI\{s+d_s \geq t\},
\]
where $\bbI$ is the indicator function.
\end{definition}
$\mathfrak{d}_t$ counts how many observations from rounds $s<t$ are still missing at the beginning of round $t$. 
Note that $\mathfrak{d}_t$ is an observable quantity, unlike the delays $d_t$. 
Therefore, $\mathfrak{d}_t$ can be used for online tuning of the learning rate. 
The learning rate under the \emph{simple tuning} is given by 
\[
\mathfrak{D}_t=\sum_{s=1}^t\mathfrak{d}_t\qquad,\qquad\eta_t^{-1}=\sqrt{2\mathfrak{D}_t/\log(k)}\,.
\]
The algorithm only uses the inverse of the learning rate. If $\mathfrak{D}_t=0$, then $\eta_t^{-1}=0$ and the algorithm is well-defined, even though $\eta_t = \infty$.

\paragraph{Advanced tuning}
In the advanced tuning, we maintain a running estimate $\tilde{\mathfrak{D}}_t$ of the cumulative delay $D_{\bar S}=\sum_{t\in\bar S}d_t$ in the optimal subset of rounds $\bar S$, whereas the regret in the remaining rounds $S = [n]\setminus \bar S$ is bounded by $|S|$ (the sense of optimality of $\bar S$ is explained in the next section). 
In order to estimate $D_{\bar S}$ we modify the quantity $\mathfrak{d}_t$ by ``skipping'' some outstanding observations.
To be precise, we keep indicator variables $a_s^t\in\{0,1\}$, where $a_s^t$ indicates whether an outstanding observation from round $s$ should still be counted at round $t$:
\[
\tilde{\mathfrak{d}}_t = \sum_{s=1}^{t-1} a_s^t\bbI\{s+d_s \geq t\}.
\] 
We define
\[
\tilde{\mathfrak{D}}_t=\sum_{s=1}^t\tilde{\mathfrak{d}}_t\qquad ,\qquad\eta_t^{-1}=\sqrt{\tilde{\mathfrak{D}}_t/\log(k)}\,.
\] 
The algorithm initially waits for all observations, but as soon as the waiting time exceeds a threshold the round is ``skipped''. 
If for $s<t$ we experience a delay $d_s > \sqrt{\tilde{\mathfrak{D}}_t/ \log(k)}$, we set $(a_s^{t'})_{t'>t}$ to $0$.
The indicators are not changed retrospectively, which means that the initial waiting time still counts toward $\tilde{\mathfrak{D}}_t$. The intuition behind advanced tuning is explained in Section~\ref{sec:intuition}.

\begin{algorithm}
\caption{Advanced tuning of $\eta_t$ for Alg.~\ref{alg:FTRL-delay}}
\label{alg:eta}
\DontPrintSemicolon
\LinesNumberedHidden
 
{\color{blue}\KwInit{$\tilde{\mathfrak{D}}_0=0$ {\normalfont and}   $(a_s^t)_{s=1,\dots,n;t=1,\dots,n}=1$}\;}
{\color{violet}\KwUpdate{$\eta_t$}{
	Set $\tilde{\mathfrak{d}}_t = \sum_{s=1}^{t-1}\bbI\{s+d_s\geq t\}a_s^t$\;
	Update $\tilde{\mathfrak{D}}_t = \tilde{\mathfrak{D}}_{t-1} + \tilde{\mathfrak{d}}_t$ \;
	Set $\eta_t^{-1} = \sqrt{\tilde{\mathfrak{D}}_t/\log(k)}$\;
	\For{$s= 1,\ldots,t-1$
	}{
		\If{$\min\{d_s,t-s\} > \eta_t^{-1}$}{
			$(a_s^{t'})_{t'>t}=0$ ~~~ (At most one index $s$ satisfies the \textbf{if}-condition, see Lemma~\ref{lem:elimination})\;
		}
	}
}}
\end{algorithm}

\section{Main results}
\label{sec:main results}
In this section, we present regret upper bounds for \cref{alg:FTRL-delay} with \emph{simple tuning} and \emph{advanced tuning}.
The first result confirms the conjecture of \citet{NGMM16} that an upper bound of $\mathcal{O}(\sqrt{kn}+\sqrt{D\log(k)})$ is achievable with a simple algorithm.
The second result shows that it is possible to obtain a refined bound of $\mathcal{O}(\sqrt{kn}+\min_{S\subset[n]}(|S|+\sqrt{D_{\bar S}\log(k)}))$ by a more careful tuning of the learning rate.

\subsection{Adaptation to the total delay $D$}


%

The following theorem provides a regret bound for Algorithm~\ref{alg:FTRL-delay} with \emph{simple tuning}.

\begin{theorem}
\label{thm:simple}
The regret of Algorithm~\ref{alg:FTRL-delay} with any non-increasing positive sequence of learning rates $(\eta_t)_{t=1,\dots,n}$ satisfies
\begin{align*}
    \Reg_n\leq 4\sqrt{kn} + \eta_n^{-1}\log(k) + \sum_{t=1}^n \eta_t \mathfrak{d}_t \,.
\end{align*}
In particular, the simple tuning $\eta_t^{-1} = \sqrt{2\mathfrak{D}_t/\log(k)} = \sqrt{2(\sum_{s=1}^t\mathfrak{d}_s)/\log(k)}$ is proper and leads to a regret bound of
\begin{align*}
    \Reg_n \leq 4\sqrt{kn} + \sqrt{8D\log(k)}\,.
\end{align*}
\end{theorem}

\begin{proof}
The first statement is a special case of Theorem~\ref{thm:rho-bound}, which is presented in \cref{sec:proofs}. 
For the second statement we use a standard summation lemma, by which for a sequence of positive $\mathfrak{d}_1,\dots,\mathfrak{d}_n$ we have $\sum_{t=1}^n \lr{\mathfrak{d}_t / \sqrt{\sum_{s=1}^t \mathfrak{d}_s}}\leq 2\sqrt{\sum_{s=1}^t \mathfrak{d}_t}$ \citep[Lemma 8]{SBCA14} and the convention that if $\mathfrak{d}_t=0$ then $\eta_t \mathfrak{d}_t = 0$ (so that zero terms naturally fall out of the summation). By substituting the definition of the learning rate in the second statement into the first statement and using the summation lemma we obtain
\[
\Reg_n \leq 4\sqrt{kn} +\sqrt{8\mathfrak{D}_n\log(k)}\,.
\]
Finally, note that an observation from round $t$ with delay $d_t$ contributes 1 to each of $\mathfrak{d}_t, \dots, \mathfrak{d}_{t+d_t}$, i.e., it contributes $d_t$ to the total sum of the number of outstanding observations $\sum_{t=1}^n \mathfrak{d}_t$. 
Since we have assumed that $t + d_t \leq n$ for all $t$, we have $\sum_{t=1}^n \mathfrak{d}_t = \sum_{t=1}^n d_t = D$.
\end{proof}

The main advantage of Algorithm~\ref{alg:FTRL-delay} and Theorem~\ref{thm:simple} compared to the work of \citet{TCS19} is that the tuning requires neither the knowledge of $D$ and $n$, nor doubling. 

\subsection{Refined bounds for unbalanced delays}

\citet{TCS19} observed that if the delays are highly unbalanced it may be worth skipping rounds with overly large delays rather than keeping them in the analysis. Let $S$ denote the set of skipped rounds and $|S|$ their number. The regret in every skipped round is trivially bounded by 1 and, assuming we knew which rounds to skip, we could reduce the regret bound to $\mathcal{O}\lr{\sqrt{kn} + |S| + \sqrt{D_{\bar S} \log(k)}}$. As shown by \citeauthor{TCS19}, this could potentially be much smaller than the regret bound in Theorem~\ref{thm:simple}. For example, if the delay in the first $\theta(\sqrt{kn})$ rounds is of order $n$ and the delay in the remaining rounds is zero, then the regret bound in Theorem~\ref{thm:simple} is of order $n^{3/4}$, whereas the refined regret bound is of order $n^{1/2}$ (ignoring the dependence on $k$). The challenge faced by \citeauthor{TCS19} was that they had to know the delays in advance (more precisely, ``at action time'') in order to tune the parameters of their algorithm and make the skipping decision. Since we have an anytime algorithm, we are able to obtain the refinement with no need in advance knowledge of the delay information. Strictly speaking, we even do not need to skip observations and we can obtain the refinement by using all observations and only adjusting the learning rate appropriately, although technically the ``no-skipping'' solution yields the same regret bound as skipping.

The following theorem provides our adaptive bound.
\begin{theorem}
\label{thm:adaptive}
\cref{alg:FTRL-delay} with advanced learning rate tuning provided in \cref{alg:eta} satisfies
\begin{align*}
\Reg_n &\leq 4\sqrt{kn}\\
&\quad + 10\max\begin{cases}\min_{S\subset[n]}|S|+\sqrt{D_{\bar S}\log(k)},\\
2\log(k).\end{cases}\,
\end{align*}
\end{theorem}
The proof is postponed to \cref{sec:proofs}

\subsection{Intuition behind the ``skipping'' procedure} 
\label{sec:intuition}

In order to give an intuition behind the refined algorithm we provide a simple back-of-the-envelope calculation. If we skip $|S|$ rounds and trivially bound their regret by 1 and apply Theorem~\ref{thm:simple} to the remaining rounds, then the regret bound is $\mathcal{O}(\sqrt{kn} + \sqrt{D_{\bar S} \log(k)} + |S|)$. Thus, the number of skipped rounds can be as large as $\sqrt{D_{\bar S} \log(k)}$ without significantly impacting the bound. Obviously, we want to skip rounds with the largest delays, but how should we determine the skipping threshold $X$? If we want to achieve a significant reduction in the regret bound, the skipped delay $D_S = \sum_{t\in S} d_t \geq X|S|$ should be at least as large as the remaining delay $D_{\bar S}$, because $D = D_S + D_{\bar S}$ and our aim is to reduce the $\sqrt{D\log k}$ term. Thus, if we put a threshold at $X$ and skip $\sqrt{D_{\bar S} \log(k)}$ rounds we want to have $X\sqrt{D_{\bar S} \log(k)} \geq D_{\bar S}$. Therefore, we aim at $X = \sqrt{D_{\bar S} / \log(k)}$. However, there are two challenges: (a) we do not know the delays $d_t$ in advance and, therefore, we do not know which rounds to skip, and (b) the threshold definition is recursive: $X$ depends on $D_{\bar S}$ and $D_{\bar S}$ depends on $X$. 

The strategy that we take in Algorithm~\ref{alg:eta} is the following: we keep a running estimate $\tilde{\mathfrak{D}}_t$ of $D_{\bar S}$. For an observation from round $s$ we initially start waiting and count it in the number of outstanding observations $\tilde{\mathfrak{d}}_t$ for the initial rounds. However, we constantly monitor the waiting time and if the observation has not arrived within $\sqrt{\tilde{\mathfrak{D}}_t/\log(k)}$ rounds we stop waiting. The initial rounds we have been waiting for still count for the estimate $\tilde{\mathfrak{D}}_t$. Another quick back-of-the-envelope calculation shows that if $\tilde{\mathfrak{D}}_t$ is indeed a good approximation of $D_{\bar S}$, then the extra delay from the initial waiting rounds is of order $\sqrt{D_{\bar S} \log(k)}\sqrt{D_{\bar S} / \log(k)} = D_{\bar S}$, where the first term is a rough estimate of the number of rounds that we skip and the second term is a rough estimate of the initial waiting time for each of the observations. Thus, the initial waiting time has no significant impact on the final bound.

\cref{alg:eta} follows this intuitive approach. 
We use indicator variables $(a_s^t)_{(s,t) \in [n]^2}$ to keep track of which observations $\ell_{s,A_s}$ we are still waiting for at round $t$ (expressed by $a_s^t = 1$) and which not (expressed by $a_s^t = 0$). We use $\tilde{\mathfrak{d}}_t$ to count the truncated number of outstanding observations, where those observations we are no longer waiting for at round $t$ are excluded from counting. We provide a detailed analysis in \cref{sec:advanced proof}, but before we get there we provide a refined version of Theorem~\ref{thm:simple}, which allows us to use all observations and only use skipping in the tuning of the learning rate. (Though, as already mentioned, complete skipping of the observations would lead to the same regret bound as in Theorem~\ref{thm:adaptive}.)

\section{Analysis of FTRL for bandits with delays}
\label{sec:proofs}
In this section we develop a novel analysis of FTRL-style algorithms, leading to Theorem~\ref{thm:rho-bound}, which generalizes the first part of \cref{thm:simple}.
The analysis is based on skipping, similar to the techniques used by \citet{TCS19}.
Afterward, we use the general regret bound to prove \cref{thm:adaptive}.  
We introduce a straightforward generalisation of the outstanding delay with ``skipping''.
\begin{definition}
For any set $S\subset[n]$, define the outstanding delay on $\bar S = [n]\setminus S$ as
\begin{align*}
    \mathfrak{d}^{\bar S}_t = \sum_{s=1}^{t-1}\bbI\{s+d_s\geq t \,\land\, s\in \bar S\}\,.
\end{align*}
\end{definition}
The cumulative outstanding delays on $\bar S$ are $\mathfrak{D}^{\bar S}_t = \sum_{s=1}^t\mathfrak{d}^{\bar S}_s$.

Next we present a general regret bound which holds for any set $S$.
\begin{theorem}
\label{thm:rho-bound}
For any set $S$, the regret of Algorithm~\ref{alg:FTRL-delay} with non-increasing positive learning rates $(\eta_t)_{t=1,\dots,n}$ satisfies
\begin{align*}
    \Reg_n\leq 4\sqrt{kn} + \eta_n^{-1}\log(k) +|S|+ \sum_{t
    \in \bar S}\eta_t\mathfrak{d}^{\bar S}_t \,.
\end{align*}
\end{theorem}
\begin{remark}
The first part of \cref{thm:simple} is a direct corollary of Theorem~\ref{thm:rho-bound} with $S=\varnothing$.
\end{remark}
The proof uses Lemmas~\ref{lem:bregman}, \ref{lem:penalty}, and \ref{lem:stability}. In order to motivate them we first present the proof and then the lemmas.
\begin{proof}
We define cumulative losses $\hat L_t=\sum_{s=1}^{t-1}\hat\ell_s$, the best arm in hindsight $i^* = \arg\min \sum_{t=1}^n \ell_{t,i}$, the cumulative observed losses $\hat L^{obs}_t = \sum_{s:s+d_s<t}\hat\ell_s$, and the cumulative losses on $\bar S$ by
$\hat L_t^{\bar S}=\sum_{s=1}^{t-1}\bbI\{s\in\bar S\,\lor\,s+d_s<t\}\hat\ell_s$ .
The latter term captures all losses that have been received up to time $t$ and additionally all outstanding losses that are not in the set $S$.
We decompose the regret into the following terms:
\begin{align*}
    &\Reg_n =\E\left[\sum_{t=1}^n \ell_{t,A_t}-\ell_{t,i^*}\right]\\
    &\leq \E\left[\sum_{t\in\bar S} \ip{x_t,\hat\ell_{t}}-\hat L_{n+1,i^*}\right] + |S|\\
    &=\E\Bigg[\underbrace{\sum_{t\in\bar S}\Bigg( \overline F_t^*(-\hat L^{obs}_t-\hat\ell_t)-\overline F_t^*(-\hat L^{obs}_t) + \ip{x_t,\hat\ell_t}\Bigg)}_{(A)} \\
    &\quad\quad +\sum_{t\in\bar S}\Bigg(\overline F_t^*(-\hat L^{obs}_t)-\overline F_t^*(-\hat L^{obs}_t-\hat\ell_t)\\
&\quad \underbrace{\qquad\qquad\qquad-\overline F_t^*(-\hat L^{\bar S}_t)+\overline F_t^*(-\hat L^{\bar S}_{t+1})\Bigg)}_{(B)}\\
    &\quad\quad +\underbrace{\sum_{t\in \bar S}\Bigg(\overline F_t^*(-\hat L^{\bar S}_{t})-\overline F_t^*(-\hat L^{\bar S}_{t+1})\Bigg) -\hat L_{n+1,i^*}}_{(C)}\Bigg]\\
    &\qquad+|S|.
\end{align*}

Term $(A)$ is a typical Bregman divergence term from the classical FTRL/OMD analysis and depends on the local norm of the regularizer.
\cref{lem:bregman} gives
\begin{align*}
\E[(A)] \leq \sum_{t=1}^n\sqrt{k}/\sqrt{t}\leq 2\sqrt{kn}\,.
\end{align*}

Term $(C)$ can also be bounded by standard techniques. 
\cref{lem:penalty} gives us
\begin{align*}
(C) \leq 2\sqrt{kn} + \eta^{-1}_n\log(k)\,.
\end{align*}

Term $(B)$ requires a novel analysis, which is presented in \cref{lem:stability}.
It allows to bound the second term by
\begin{align*}
\E[(B)] \leq \sum_{t\in\bar S} \eta_t\mathfrak{d}^{\bar S}_t\,.
\end{align*}
Combining everything finishes the proof.
\end{proof}

\paragraph{Support lemmas for the proof of \cref{thm:rho-bound}}
The proofs for all the support lemmas are given in the supplementary material, \cref{app:main lemmas}.
The first Lemma is a small modification of the classical result that bounds the Bregman divergence by the local norm of the regularizer.
We show that we can bound the local norm by the contribution of the Tsallis entropy.
\begin{lemma}
\label{lem:bregman}
For any $t$ it holds that
\begin{align*}
\E\left[\overline F_t^*(-\hat L^{obs}_t-\hat\ell_t)-\overline F_t^*(-\hat L^{obs}_t) + \ip{x_t,\hat\ell_t}\right]\leq \frac{\sqrt{k}}{\sqrt{t}}\,.
\end{align*}
\end{lemma}
The second Lemma bounds the so-called ``penalty'' term coming from the regularization penalty. It appears in almost identical form in the literature \citep[Exercise 28.12]{LS19bandit-book}.
\begin{lemma}
\label{lem:penalty}
For any non-increasing learning rate $\eta_t$, it holds that
\begin{multline*}
\sum_{t\in\bar S}\Bigg(\overline F_t^*(-\hat L^{\bar S}_{t})-\overline F_t^*(-\hat L^{\bar S}_{t+1}) \Bigg)-\hat L_{n+1,i^*} \qquad
\\\leq 2\sqrt{kn}+\eta_n^{-1}\log(k)\,.
\end{multline*}
\end{lemma}
The third quantity does not show up in the regular analysis without delays.
We show that similarly to the Bregman divergence, it depends on the local norm of the regularizer.
However, it is beneficial to use the norm of the negative entropy instead of the Tsallis entropy.
\begin{lemma}
\label{lem:stability}
For any $t$ it holds that
\begin{align*}
\E\Bigg[\overline F_t^*&(-\hat L^{obs}_t)-\overline F_t^*(-\hat L^{obs}_t-\hat\ell_t)\\
&-\overline F_t^*(-\hat L^{\bar S}_t)+\overline F_t^*(-\hat L^{\bar S}_{t+1})\Bigg]\leq \eta_t\mathfrak{d}^{\bar S}_t\,.
\end{align*}
\end{lemma}
\subsection{Refined regret bound}
\label{sec:advanced proof}
With a suitable learning rate, \cref{thm:rho-bound} leads directly to the bound
\begin{align*}
\Reg_n \leq 4\sqrt{kn} + |S|+2\sqrt{\sum_{t\in\bar S}d_t\log(k)}\,.
\end{align*}
Below, in the proof of \cref{thm:adaptive}, we show that the learning rate in \cref{alg:eta} brings us within a multiplicative constant factor of the minimum of the above bound, $4\sqrt{kn} + \min_S (|S|+2\sqrt{D_{\bar S}\log(k)})$.

From now on, let $S$ be the set 
\[S = \{t\in[n]\,|\,a_t^n = 0\}\,,\]
which is the set of rounds ``skipped'' by Algorithm~\ref{alg:eta}. Since $(a_s^t)_{t=1,\dots,n}$ is non-increasing, we have for any $t\in \bar S$:
$\mathfrak{d}^{\bar S}_t \leq \tilde{\mathfrak{d}}_t$.
Furthermore, the following lemma bounds the magnitude of $|S|$:
\begin{lemma}
\label{lem:skipping bound}
For any sequence of delays $d_t$, \cref{alg:eta} satisfies
\begin{align*}
|S| = \sum_{t=1}^n\bbI\{a_t^n=0\} \leq 2\sqrt{\tilde{\mathfrak{D}}_n\log(k)}\,.
\end{align*}
\end{lemma}
The proof is provided in the supplementary material, \cref{app:eta lemmas}.

Finally we have all the prerequisites to prove \cref{thm:adaptive}.

\begin{proof}[Proof of Theorem~\ref{thm:adaptive}]
Using \cref{thm:rho-bound} and \cref{lem:skipping bound}, we have
\begin{align*}
\Reg_n &\leq 4\sqrt{kn} + \eta_n^{-1}\log(k) +|S|+ \sum_{t\in\bar S}\eta_t\mathfrak{d}^{\bar S}_t\\
&\leq 4\sqrt{kn} + \eta_n^{-1}\log(k) + |S|+\sum_{t\in\bar S}\eta_t\tilde{\mathfrak{d}}_t\\
&\leq 4\sqrt{kn} + 5\sqrt{\tilde{\mathfrak{D}}_n\log(k)}\,.
\end{align*}
Now we need to control the term $\sqrt{\tilde{\mathfrak{D}}_n\log(k)}$.
Let's consider the case $\tilde{\mathfrak{D}}_n \leq 4 \sqrt{\tilde{\mathfrak{D}}_n\log(k)}$, then $\sqrt{\tilde{\mathfrak{D}}_n\log(k)}\leq 4 \log(k)$ and we are done.
Otherwise, define $\tilde d_t = \sum_{s=t+1}^{t+d_t}a_{t}^s$, i.e., the contribution of round $t$ to the sum $\tilde{\mathfrak{D}}_n$.
Then we can decompose 
\begin{align*}
\tilde{\mathfrak{D}}_n &= \sum_{s=1}^n\sum_{t<s}\bbI\{t+d_t>s\}a_t^s\\
&= \sum_{t=1}^n\sum_{s>t}\bbI\{t+d_t > s\}a_t^s\\
&= \sum_{t=1}^n\sum_{s=t+1}^{t+d_t}a_t^s
=\sum_{t=1}^n \tilde d_t\,.
\end{align*}
Any element $t\in \bar S$ satisfies 
\begin{align*}
\tilde d_t\leq \sqrt{\tilde{\mathfrak{D}}_t / \log(k)}\leq \sqrt{\tilde{\mathfrak{D}}_n / \log(k)}\,,
\end{align*}
while any element $t\in S$ satisfies
\begin{align*}
\tilde d_t&\leq \left\lceil\sqrt{\tilde{\mathfrak{D}}_t / \log(k)}\right\rceil \leq \left\lceil\sqrt{\tilde{\mathfrak{D}}_n / \log(k)}\right\rceil\\
& \leq \sqrt{\tilde{\mathfrak{D}}_n / \log(k)} + 1\,.
\end{align*}
Therefore, we can bound for any $R\subset[n]$:
\begin{align*}
\sum_{t\in\bar R} d_t &\geq \sum_{t\in\bar R} \tilde d_t \geq \tilde{\mathfrak{D}}_n - |R|\sqrt{\tilde{\mathfrak{D}}_n / \log(k)} - |S|\\
&\geq \tilde{\mathfrak{D}}_n - |R|\sqrt{\tilde{\mathfrak{D}}_n / \log(k)} - 2\sqrt{\tilde{\mathfrak{D}}_n\log(k)} \\
&\geq \frac{1}{2}\tilde{\mathfrak{D}}_n - |R|\sqrt{\tilde{\mathfrak{D}}_n / \log(k)}\,.
\end{align*}

This implies that 
\begin{align*}
&\min_{R\subset[n]} |R| +\sqrt{\sum_{t\in\bar R} d_t\log(k)} \\
&\geq \min_{r\in [0,\frac{1}{2}\sqrt{\tilde{\mathfrak{D}}_n\log(k)}]} r +\sqrt{\frac{1}{2}\tilde{\mathfrak{D}}_n\log(k) - r\sqrt{\tilde{\mathfrak{D}}_n \log(k)}}\,.
\end{align*}
The function is concave in $r$ so the minimum is achieved at one of the endpoints of the interval, which happens to be
$r=\frac{1}{2}\sqrt{\tilde{\mathfrak{D}}_n\log(k)}$ for which the function equals $\frac{1}{2}\sqrt{\tilde{\mathfrak{D}}_n\log(k)}$. 
Hence, we have shown
\begin{align*}
\sqrt{\tilde{\mathfrak{D}}_n\log(k)} \leq 2\min_{R\subset [n] } \left(|R|+\sqrt{\sum_{s\in\bar R}d_s\log(k)}\right)\,,
\end{align*}
which concludes the proof.
\end{proof}

\section{Discussion}
\label{sec:conclusion}
We confirmed an open conjecture from \citet{NGMM16} by presenting a simple FTRL algorithm for adversarial bandits with arbitrary delays and proving regret upper bound that matches the lower bound within constants.
Furthermore, we proposed a refined tuning of the learning rate that achieves even tighter regret bound for highly unbalanced delays.
We strictly improve on the state-of-the-art bounds and present the first anytime result requiring no doubling, skipping, or advance information about the delays.

If the delays are all $0$, then our algorithm reduces to the Tsallis-INF algorithm of \citet{ZS19}, which has been proven to be simultaneously optimal in both the stochastic and the adversarial setting.
We conjecture that the algorithm presented in this paper is capable of obtaining logarithmic regret in the stochastic setting, but leave the analysis for future work.

Another open question is the tightness of our adaptive bound $\mathcal{O}(\sqrt{kn}+\min_{S\subset[n]}(|S|+\sqrt{D_S\log(k)}))$.
We conjecture that for a fixed set of delays $\{d_1,\dots,d_n\}$ which the adversary is allowed to permute without changing the magnitudes,
the upper bound is actually tight. 

\subsubsection*{Acknowledgements}

We would like to thank Andr{\' a}s Gy{\" o}rgy and Tobias Sommer Thuner for fruitful discussions and Riccardo Della Vecchia for spotting a mistake in an earlier version of this paper. The work was partly supported by the Independent Research Fund Denmark, grant number 9040-00361B.

\bibliographystyle{plainnat}
\bibliography{all}
\ifsup
\newpage
\onecolumn
\section{SUPPLEMENTARY MATERIAL}

\subsection{Auxiliary lemmas for \cref{alg:eta}}
\label{app:eta lemmas}
\begin{lemma}
\label{lem:elimination}
\cref{alg:eta} will not deactivate more than 1 point at a time.
\end{lemma}
By \emph{deactivating} we mean setting $a_t^n=0$.
\begin{proof}
We prove the lemma by contradiction.
Assume that $s_1,s_2$ are both deactivated at time $t$.
W.l.o.g.\ let $s_2 \leq s_1 -1$.
Deactivation of $s_1$ at time $t$ means $t-s_1 \geq \sqrt{\mathcal{\mathfrak{D}}_t/\log(k)}\geq \sqrt{\mathcal{\mathfrak{D}}_{t-1}/\log(k)}$.
At the same time we assumed $t-1-s_2 \geq t-s_1$, which means that $s_2$ would have been deactivated at round $t-1$ or earlier.
\end{proof}

\begin{proof}[Proof of \cref{lem:skipping bound}]
Recall that $\tilde d_t = \sum_{s=t+1}^{t+d_t}a_{t}^s$ is the contribution of a timestep $t$ to the sum $\tilde{\mathfrak{D}}_n$.

Let $(t_1,\dots,t_{|S|})$ be an indexing of $S$. By \cref{lem:elimination} we deactivate at most one
$a_{t_m}^n$ per round. Thus, we have that
\begin{align*}
&\tilde{d}_{t_m} > \sqrt{\tilde{\mathfrak{D}}_{t_m+d_{t_m}}/\log(k)}\geq \sqrt{\sum_{i=1}^m\tilde{d}_{t_i}/\log(k)} = \frac{\sqrt{\tilde d_m + \sum_{i=1}^{m-1} \tilde d_{t_i}}}{\sqrt{\log(k)}}\,.
\end{align*}
By solving the quadratic inequality in $d_{t_m}$ we obtain
\begin{align*}
\tilde{d}_{t_m} > \frac{1 + \sqrt{1+4\log(k)\sum_{i=1}^{m-1}\tilde{d}_{t_i}}}{2\log(k)} \,.
\end{align*}
Now we prove by induction that $\tilde{d}_{t_m} > \frac{m}{2\log(k)}$.
The induction base holds since $\tilde{d}_{t_1} = 1$.
For the inductive step we have
\begin{align*}
\tilde{d}_{t_m} > \frac{1 + \sqrt{1+4\log(k)\sum_{i=1}^{m-1}\tilde{d}_{t_i}}}{2\log(k)}
> \frac{1 + \sqrt{1+m(m-1)}}{2\log(k)} > \frac{m}{2\log(k)}\,.
\end{align*}
Finally, we have
\begin{align*}
\sqrt{\tilde{\mathfrak{D}}_n\log(k)} \geq \sqrt{\sum_{m=1}^{|S|}\tilde{d}_{t_m}\log(k)} > \sqrt{\frac{|S|(|S|+1)}{4}} > \frac{1}{2}|S|\,.
\end{align*}
\end{proof}

\subsection{Standard properties of FTRL analysis}
\label{app:main properties}
First we list some standard properties of FTRL that we use in the proofs of the remaining lemmas. We recall that $f_t(x)=-2\sqrt{t}\sqrt{x}+\eta_t^{-1} x\log(x)$.
\begin{claim}
$f''_t(x):\R_+\rightarrow \R_+$ are monotonically decreasing functions and ${f^*}_t':\R\rightarrow \R_+$ are convex and monotonically increasing.
\end{claim}
\begin{proof}
By definition $f''_t(x) = \frac{1}{2}\sqrt{t}x^{-3/2}+\eta_t^{-1}x^{-1}$, which concludes the first statement.
Since $f_t$ are Legendre functions, we have ${f_t^*}''(y) = f_t''({f_t^*}'(y))^{-1} > 0$. 
Therefore the function is monotonically increasing.
Since both $f_t''(x)^{-1}$, as well as ${f_t^*}'(y)$ are increasing, the composition is as well and ${f_t^*}'''>0$.
\end{proof}
\begin{claim}
\label{clm:shift}
For any convex $F$, for $L\in\R^k$ and $c\in\R$:
\begin{align*}
    \overline F^*(L+c\ones) = \overline F^*(L)+c\,.
\end{align*}
\end{claim}
\begin{proof}
By definition $\overline F^*(L+c\ones) = \max_{x\in\Delta([k])}\ip{x,L+c\ones}-F(x)=\max_{x\in\Delta([k])}\ip{x,L}-F(x)+c=\overline F^*(L)+c$.
\end{proof}
\begin{claim}
\label{clm:gradient}
For any $x_t$ there exists $c\in\R$, such that:
\begin{align*}
    x_t=\nabla\overline F_t^*(-\hat L^{obs}_t) = \nabla F_t^*(-\hat L^{obs}_t+c\ones)=\nabla F_t^*(\nabla F_t(x_t))\,.
\end{align*}
\end{claim}
\begin{proof}
By the KKT conditions, there exists $c\in\R$, such that $x_t = \argmax_{x\in\Delta([k])}\ip{x,-\hat L^{obs}_t}+F_t(x)$ satisfies $\nabla F_t(x_t) = -\hat L^{obs}_t+c\ones$. The rest follows by the standard property $\nabla F = (\nabla F^*)^{-1}$ of Legendre $F$. 
\end{proof}
\begin{claim}
\label{clm:unconstrained}
For any Legendre function $F$ and $L\in\R^k$ it holds that
\begin{align*}
    \overline F^*(L) \leq F^*(L)\,
\end{align*}
with equality iff there exists $x\in\Delta([k])$, such that $L=\nabla F(x)$.
\end{claim}
\begin{proof}
The first statement follows from the definition, since for any $A\subset B$: $\max_{x\in A}f(x) \leq \max_{x\in B}f(x)$.
The second part follows because equality means that $\argmax_x \ip{x,L}-F(x)=\nabla F^*(L)\in \Delta([k])$, which is equivalent to the statement.
\end{proof}
\begin{claim}
\label{clm:bound by unconstrained}
	For any $x\in\Delta([k])$, $L\geq 0$ and $i\in[k]$:
\begin{align*}
\nabla\overline F_t^*(\nabla F_t(x)-L)_i \geq \nabla F_t^*(\nabla F_t(x)-L)_i\,.
\end{align*}
\end{claim}
\begin{proof}
	By \cref{clm:gradient}, there exists $c\in\R:\, \nabla\overline F_t^*(\nabla F_t(x)-L)=\nabla F_t^*(\nabla F_t(x)-L+c\ones)$.
	The statement is equivalent to $c$ being non-negative, since ${f^*}'$ are monotonically increasing.
	If $c< 0$, then  
\begin{align*}
1=\sum_{i=1}^k(\nabla \overline F_t^*(\nabla F_t(x)-L))_i=\sum_{i=1}^k(\nabla F_t^*(\nabla F_t(x)-L+c\ones))_i= \sum_{i=1}^k{f_t^*}'(f_t'(x_i)-L_i+c) < \sum_{i=1}^k{f_t^*}'(f_t'(x_i)) = 1\,,
\end{align*}
which is a contradiction and completes the proof.
\end{proof}
\begin{claim}
\label{clm:bregman bound}
Let $D_F(x,y) = F(x)-F(y)-\langle x-y,\nabla F(y)\rangle$ be the Bregman divergence of a function $F$.
For any Legendre function $f$ with monotonically decreasing second derivative, $x\in\dom(f)$, and $\ell \geq 0$, such that $f'(x)-\ell\in\dom(f^*)$:
\begin{align*}
    D_{f^*}(f'(x)-\ell,f'(x))\leq \frac{\ell^2}{2f''(x)}\,.
\end{align*}
\end{claim}
\begin{proof}
By Taylor's theorem, there exists $\tilde x \in [{f^*}'(f'(x)-\ell),x]$, such that 
$D_{f^*}(f'(x)-\ell,f'(x))= \frac{\ell^2}{2f''(\tilde x)}$. $\tilde x$ is smaller than $x$, since ${f^*}'$ is monotonically increasing. Finally, using the fact that the second derivative is decreasing allows to bound $f''(\tilde x)^{-1}\leq f''(x)^{-1}$.
\end{proof}
\begin{claim}
\label{clm:monotony}
For any convex $F$, $L_2\geq L_1$ (coordinate wise), it holds
\begin{align*}
    \overline F^*(-L_1) \geq \overline F^*(-L_2)\,.
\end{align*}
\end{claim}
\begin{proof}
\begin{align*}
    \overline F^*(-L_2)&=\ip{\nabla \overline F^*(-L_2),-L_2} +F(\nabla \overline F^*(-L_2))\\
    &\leq \ip{\nabla \overline F^*(-L_2),-L_1} +F(\nabla \overline F^*(-L_2))\\
    &\leq \max_{x\in\Delta([k])}\ip{x,-L_1} +F(x)\\
    &=\overline F^*(-L_1)\,.
\end{align*}
\end{proof}

\subsection{Proofs of the Main Lemmas}
\resetnotes
\label{app:main lemmas}
\begin{proof}[Proof of Lemma~\ref{lem:bregman}]

\begin{align*}
\overline F_t^*(-\hat L^{obs}_t-\hat\ell_t)-\overline F_t^*(-\hat L^{obs}_t) + \ip{x_t,\hat\ell_t}
&\overset{\eqnote{
Applies \cref{clm:shift,clm:gradient}.
}}{=}\overline F_t^*(\nabla F_t(x_t)-\hat\ell_t)-\overline F_t^*(\nabla F_t(x_t)) + \ip{x_t,\hat\ell_t}\\
&\overset{\eqnote{
Follows from both parts of \cref{clm:unconstrained}.
}}{\leq}  F_t^*(\nabla F_t(x_t)-\hat\ell_t)- F_t^*(\nabla F_t(x_t)) + \ip{x_t,\hat\ell_t}\\
&= \sum_{i=1}^k D_{f_t^*}(f'_t(x_{t,i})-\hat\ell_{t,i},f'_t(x_{t,i}))\\
&= D_{f_t^*}(f'_t(x_{t,A_t})-\ell_{t,A_t}x_{t,A_t}^{-1},f'_t(x_{t,A_t}))\\
&\overset{\eqnote{
Uses \cref{clm:bregman bound}.
}}{\leq} \frac{1}{2}\ell_{t,A_t}^2x_{t,A_t}^{-2}f''_t(x_{t,A_t})^{-1}\\
&\leq \frac{1}{2}\ell_{t,A_t}^2x_{t,A_t}^{-2}f''_{t1}(x_{t,A_t})^{-1}\\
&= \frac{1}{2}\ell_{t,A_t}^2x_{t,A_t}^{-2}\frac{2x_{t,A_t}^\frac{3}{2}}{\sqrt{t}}\\
&\leq \frac{x_{t,A_t}^{-\frac{1}{2}}}{\sqrt{t}}\,.
\end{align*}
\writeeqnotes
In expectation we get 
\begin{align*}
&\E\left[\overline F_t^*(-\hat L^{obs}_t-\hat\ell_t)-\overline F_t^*(-\hat L^{obs}_t) + \ip{x_t,\hat\ell_t}\right]
\leq \sum_{i=1}^k\frac{\sqrt{x_{t,i}}}{\sqrt{t}}\leq \frac{\sqrt{k}}{\sqrt{t}}\,.
\end{align*}
\end{proof}
\resetnotes 

\begin{proof}[Proof of Lemma~\ref{lem:penalty}]
Let $\tilde x_t = \argmax_{x\in\Delta([k])} \ip{x,-\hat L^{\bar S}_{t}}-F_t(x)$, then
\begin{align*}
\overline F_t^*(-\hat L^{\bar S}_{t}) = \ip{\tilde x_t,-\hat L^{\bar S}_{t}}-F_t(\tilde x_t).
\end{align*}
Furthermore, since $\overline F^*(-\hat L^{\bar S}_{t}) = \max_{x\in\Delta([k])} \ip{x,-\hat L^{\bar S}_{t}}-F(x)$, we have
\begin{align*}
-\overline F_{t-1}^*(-\hat L^{\bar S}_{t}) &\leq -\ip{\tilde x_t,-\hat L^{\bar S}_{t}}+F_{t-1}(\tilde x_t)\,,\\
-\overline F_n^*(-\hat L^{\bar S}_{n+1}) &\leq -\ip{\e_{i^*},-\hat L^{\bar S}_{n+1}} +F_n(\e_{i^*}) \leq \hat L^{\bar S}_{n+1,i^*} = \hat L_{n+1,i^*}\,.
\end{align*}
Plugging these inequalities into the LHS leads to
\begin{align*}
\sum_{t\in\bar S}\Bigg(\overline F_t^*(-\hat L^{\bar S}_{t})-\overline F_t^*(-\hat L^{\bar S}_{t+1}) \Bigg)-\hat L_{n+1,i^*}
&\leq \sum_{t=1}^n\Bigg(\overline F_t^*(-\hat L^{\bar S}_{t})-\overline F_t^*(-\hat L^{\bar S}_{t+1}) \Bigg)-\hat L_{n+1,i^*}\\
&\leq -F_1(\tilde x_1)+\sum_{t=2}^n F_{t-1}(\tilde x_t)-F_{t}(\tilde x_t)\\
&\leq \max_{x\in\Delta([k])}-F_1(x)+\sum_{t=2}^n \max_{x\in\Delta([k])}(F_{t-1}(x)-F_{t}(x))\\
&= -F_n(\ones/k)\\
&= 2\sqrt{kn}+\eta_n^{-1}\log(k)\,,
\end{align*}
where the first step follows by \cref{clm:monotony}.
\end{proof}

\begin{proof}[Proof of Lemma~\ref{lem:stability}]
Note that by definition, for any $t\in \bar S$:
\begin{align*}
    \hat L^{\bar S}_{t+1} = \hat L^{\bar S}_t + \hat\ell_t + \sum_{s\in S: s+d_s=t}\hat\ell_s\geq \hat L^{\bar S}_t + \hat\ell_t\,.
\end{align*}
So by \cref{clm:monotony}: $\overline F_t^*(-\hat L^{\bar S}_{t+1})\leq \overline F_t^*(-\hat L^{\bar S}_{t}-\hat\ell_t)$.
We define $\hat L^{miss}_t = \hat L^{\bar S}_t-\hat L^{obs}_t$.
Then we have for any $t\in\bar S$
\begin{align*}
 -\overline F_t^*(-\hat L^{\bar S}_t)+\overline F_t^*(-\hat L^{\bar S}_{t}-\hat\ell_t)
&\overset{
\eqnote{
By the fundamental theorem of calculus.
}
}{=}-\int_{0}^1 \ip{\hat\ell_t, \nabla\overline F^*_t(-\hat L^{\bar S}_t - x\hat\ell_t)}\,dx\\
&~=\,-\int_{0}^1 \ip{\hat\ell_t, \nabla\overline F^*_t(-\hat L^{obs}_t-\hat L^{miss}_t - x\hat\ell_t)}\,dx\\
&\overset{
\eqnote{
Follows from the fact that $\nabla\overline F^*_t(-L)_{A_t}$ decreases if the loss in coordinates other than $A_t$ is reduced.
}
}{\leq}-\int_{0}^1 \ip{\hat\ell_t, \nabla\overline F^*_t(-\hat L^{obs}_t-\hat L^{miss}_{t,A_t}\e_{A_t} - x\hat\ell_t)}\,dx\,.
\end{align*}
\writeeqnotes
Therefore, we have for any $t\in\bar S$
\begin{align*}
\overline F_t^*(-\hat L^{obs}_t)-\overline F_t^*(-\hat L^{obs}_t-\hat\ell_t)&-\overline F_t^*(-\hat L^{\bar S}_t)+\overline F_t^*(-\hat L^{\bar S}_{t+1})\\
&\overset{\eqnote{
uses the Fundamental theorem of calculus together with the inequality above.
}}{\leq} \int_{0}^1 \ip{\hat\ell_t,\nabla \overline F_t^*(-\hat L^{obs}_t-x\hat\ell_t)}\,dx -\int_{0}^1 \ip{\hat\ell_t, \nabla\overline F^*_t(-\hat L^{obs}_t-\hat L^{miss}_{t,A_t}\e_{A_t} - x\hat\ell_t)}\,dx\\
&\overset{\eqnote{
substitutes $\tilde{z}(x) = \nabla \overline F_t^*(-\hat L^{obs}_t-x\hat\ell_t)$ and applies \cref{clm:gradient}.
}}{=} \int_{0}^1 \ip{\hat\ell_t,\tilde z(x)- \nabla\overline F^*_t(\nabla F_t(\tilde{z}(x))-\hat L^{miss}_{t,A_t}\e_{A_t})}\,dx\\
&\overset{\eqnote{
applies \cref{clm:bound by unconstrained}.
}}{\leq} \int_{0}^1 \ip{\hat\ell_t,\tilde z(x)- \nabla F^*_t(\nabla F_t(\tilde{z}(x))-\hat L^{miss}_{t,A_t}\e_{A_t})}\,dx\\
&= \int_{0}^1 \hat\ell_{t,A_t}(\tilde z_{A_t}(x)- {f^*}'_t(f'_t(\tilde{z}_{A_t}(x))-\hat L^{miss}_{t,A_t})\,dx\\
&\overset{\eqnote{
${f^*}'(t)$ is convex, so $-{f^*}'(f'(\tilde z_{A_t}) - \ell)\leq -\tilde z_{A_t} + {f^*}''(f'(\tilde z_{A_t}))$.
}}{\leq} \int_{0}^1 \hat\ell_{t,A_t}({f^*}''_t(f'_t(\tilde{z}_{A_t}(x)))\hat L^{miss}_{t,A_t}\,dx\\
&=\int_{0}^1 \ell_{t,A_t}x_{t,A_t}^{-1}f''_t(\tilde{z}_{A_t}(x))^{-1}\hat L^{miss}_{t,A_t}\,dx\\
&\overset{\eqnote{
follows because $\tilde z_{A_t}\leq x_{t,A_t}$ and $f_t''(x)^{-1}$ is monotonically increasing.
}}{\leq}\int_{0}^1 \ell_{t,A_t}x_{t,A_t}^{-1}f''_t(x_{t,A_t})^{-1}\hat L^{miss}_{t,A_t}\,dx\\
&\leq  x_{A_t}^{-1}f_{t2}''(x_{A_t})^{-1}\hat L^{miss}_{t,A_t}\\
& = \eta_t\hat L^{miss}_{t,A_t}\,.
\end{align*}
\writeeqnotes
Finally, due to the unbiasedness of the loss estimators we have in expectation
\begin{align*}
\E[\hat L^{miss}_{t,A_t}] = \sum_{s<t}\bbI\{s+d_s\geq t \wedge s\in \bar S\}\ell_{s,A_t} \leq \mathfrak{d}^{\bar S}_t\,.
\end{align*}
\end{proof}

\fi
\end{document}